\documentclass{article}                     
\usepackage{graphicx}
 \usepackage{mathptmx}      
\newcommand{\w}{ \bf }

\usepackage[cmex10]{amsmath}
\usepackage{amsmath}
\usepackage{amsfonts}
\usepackage{amssymb}
\usepackage{amsthm}

\theoremstyle{definitions}
\newtheorem{theorem}{Theorem}[section]
\newtheorem{lemma}{Lemma}[section]
\newtheorem{proposition}{Proposition}[section]
\newtheorem{observation}{Observation}[section]

\usepackage{booktabs}

\usepackage{algorithm}
\usepackage{algorithmic}

\newcommand\minusOne{{_{-}}}
\newcommand\plusOne{{_{+}}}

\numberwithin{equation}{section}
\newcommand\prob{\mathrm{prob}}
\newcommand\for{\mbox{ for }}

\newcommand\R{\mathbb{R}}
\newcommand\Sm{S}

\newcommand\cov{\mathrm{cov}}

\newcommand\nor{\mathcal{N}}
\newcommand{\vrbf}[1]{C$_{#1}$RBF}
\newcommand{\makeitsmaller}[1]{\mbox{$#1$}}

\begin{document}

\title{Cluster based RBF Kernel\\for Support Vector Machines
}


\author{Wojciech Marian Czarnecki,
        Jacek Tabor\\
\small{Faculty of Mathematics and Computer Science,}\\
\small{Jagiellonian Unviersity, Krakow, Poland.}\\
\small{\{wojciech.czarnecki, jacek.tabor\}@uj.edu.pl}
}

\date{}       



\maketitle

\begin{abstract}
In the classical Gaussian SVM classification we use the feature space 
projection transforming points to 
normal distributions with fixed covariance matrices
(identity in the standard RBF and the covariance of the whole dataset in Mahalanobis RBF). 
In this paper we add additional information
to Gaussian SVM by considering local geometry-dependent
feature space projection. We emphasize that 
our approach is in fact an algorithm for a
construction of the new Gaussian-type kernel.

We show that better (compared to standard
RBF and Mahalanobis RBF) classification results are obtained in the simple case when
the space is preliminary divided by k-means into
two  sets and points are represented as normal distributions
with a covariances calculated according to the dataset partitioning.
We call the constructed method \vrbf{k}, where $k$
stands for the amount of clusters used in k-means.
We show empirically on nine datasets from UCI repository that
\vrbf{2}  increases the stability
of the grid search (measured as the probability of finding good parameters).
\end{abstract}

\section{Introduction}

%

In most classical machine learning models we exploit the global statistical properties of the data without analysis of their exact local geometry (SVM, Neural Networks). On the other hand -- density based methods (Bayes, EM) -- are conceptually different approaches which often lead to completely local decision criteria. 

Our approach belongs to the hybrid approaches \cite{Hinton:2006:FLA:1161603.1161605, Ma:2011:CSU:2283516.2283629} which typically try to combine supervised and unsupervised techniques in one, uniform model. Some of these methods are combinations of clustering and classification techniques in either direct form \cite{clusteredSVM} or using the complex, hierarchical structures of alternating algorithms \cite{PodolakRomanCI2013}.

In this paper, we introduce the kernel building method which exploits the local data geometry using cluster-based space partition and includes it in the constructed feature space projection. We show that even the use of k-means (with $k=2$) for the partitioning part gives interesting results. Our approach can be seen as a special case of the metric learning problem which does not change the formulation of the optimization problem being solved. As a result we give a simple and efficient method which can be easily used with existing implementations of SVM.

The constructed approach not only gives better classification results then SVM but, what is of crucial importance from the practical point of view, ,,good'' results are obtained with less complex tuning required as compared to the usage of RBF or Mahalanobis RBF kernels (see Figure~\ref{fig:grids} for example of grid search results on australian dataset). It is worth noting that proposing models and methods which reduce the complexity of metaparameters tuning is of crucial importance for practical applications. One the one hand, such optimziation can be too expensive (hierarchical, ensemble based classification~\cite{PodolakRomanCI2013}, active learning scenarios~\cite{settles2012active}) and on the other researchers from other disciplines often ignore its importance~\cite{Wang2008hergSVM,Cong2009TACESVM,Sakiyama2008stabilitySVM}.

\begin{figure}[htb]
\centering
\includegraphics[width=2.5cm]{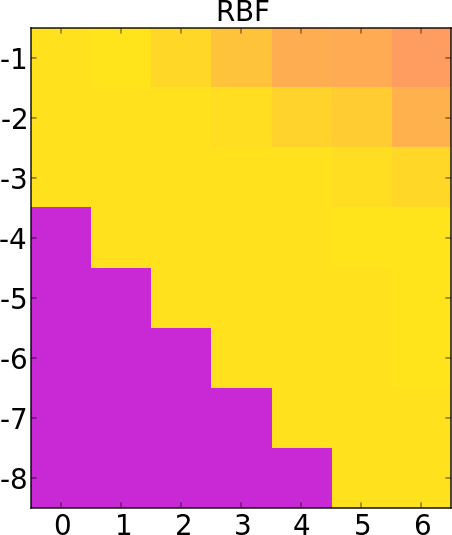}
\includegraphics[width=2.5cm]{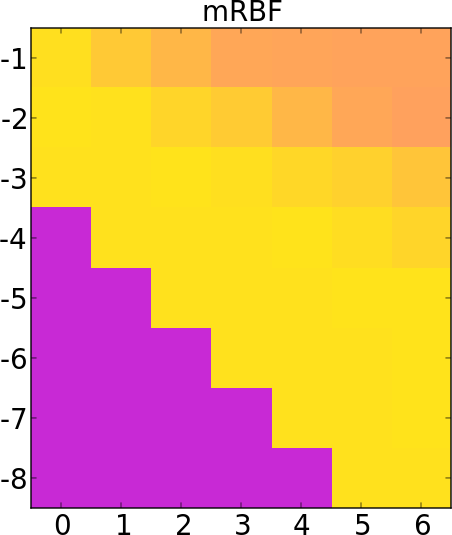}
\includegraphics[width=2.5cm]{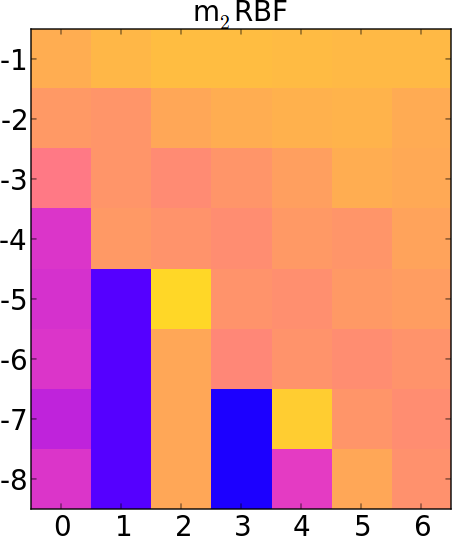}
\includegraphics[width=2.5cm]{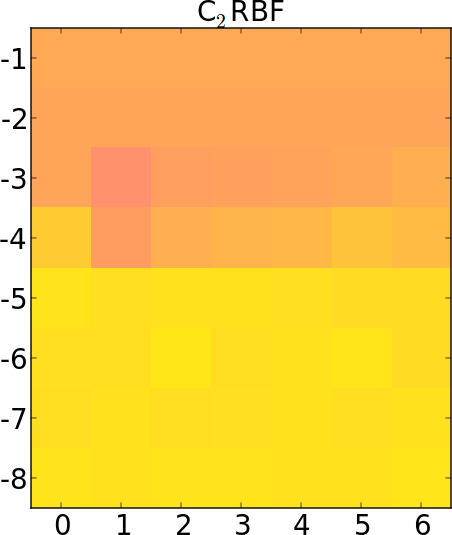}\\ \smallskip
\includegraphics[width=9cm]{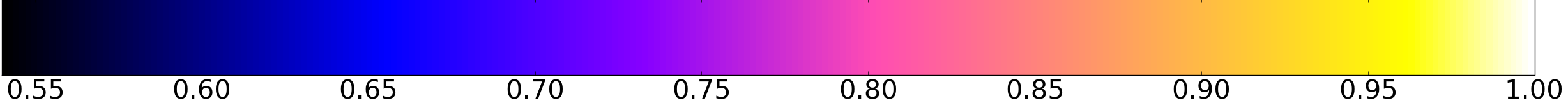}
\caption{Accuracy scores for grid search of parameters $C$ and $\gamma$, from left: simple RBF kernel, Mahalanobis kernel, k-means + Mahalanobis kernel, our method}
\label{fig:grids}
\end{figure}

We also introduce an $P_f(\alpha)$ index  (described in detail in evaluation section) able to measure how easy is to tune the model which requires some set of metaparameters and use it to evaluate our approach. It can be used both for visualization of this characteristics (similarly to how ROC curves visualize clasifier's accuracy) and for comparision of different models (similarly to AUC ROC measure). For examples one can refer to Fig.~\ref{fig:prob} in the evaluation section.

Paper is structured as follows: first, we describe our method and prove that it leads to a valid Mercer's kernel. Then we analyze some practical issues connected with algorithm implementation and usage and conclude with comparative empirical evaluation performed using datasets from UCI database.

\section{Related work}

In the recent years there is a growing interest in the fields of metric learning \cite{weinberger09distance}. Among others, Mahalanobis metric learning for the RBF SVM has been proposed \cite{learningMahalanobisRBF}. More computationally feasible solutions, which are similarly justified include performing preprocessing step. One of such approaches is a search for the smallest volume bounding ellipsoid \cite{SmallVolume} which is used to define the Mahalanobis kernel. Our approach is similar to this idea as it also performs a preprocessing in order to find some data characteristics but instead of optimization procedure we use cheap clustering technique.

%

Many researchers investigated possible fusion of \mbox{k-means} and SVMs -- these approaches spans from using k-means to reduce the training set size \cite{ReductionTraining} through reduction of support vectors count \cite{ReductionSV} to even incorporating the process of finding centroids directly into the optimization problem \cite{clusteredSVM}. However, in our work k-means is used in a completely different manner, only as a selection method for the  partition. Instead of reducing amount of available information (by either removing training samples or support vectors) it introduces additional kind of knowledge into the process.

%

Another branch of related approaches are ensembles-based models. Recently proposed HCOC \cite{PodolakRomanCI2013,jackowski2014improved} model alternates between classification and clustering steps on many levels of tree-like structure. It also exploits some additional kind of knowledge -- which classes are the most likely to be confused. Other authors also showed that clustering can be used to simply divide the problem into smaller ones solved independently by a separate classifiers \cite{Hybridnsamble,ClusterEnsamble}. In \vrbf{k}, instead of splitting data and analyzing the output of several classifiers, we propose to include all the information gained through clustering into one, generic classifier.


\section{\vrbf{k} kernel building algorithm}

\begin{figure*}[htb]
\begin{center}

 \includegraphics[width=0.3\textwidth]{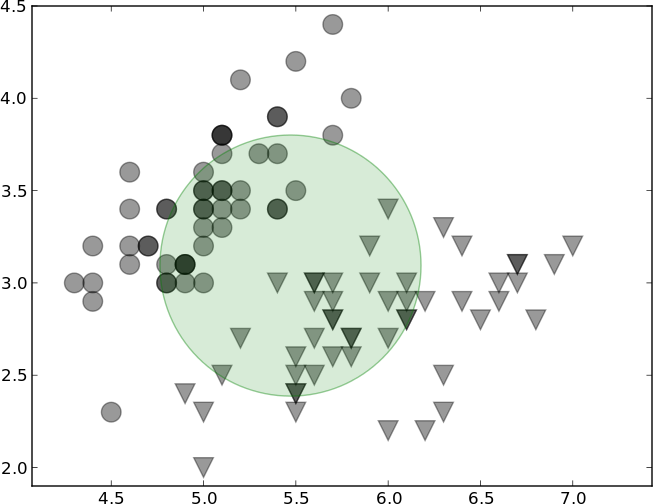}\,\,\,
 \includegraphics[width=0.3\textwidth]{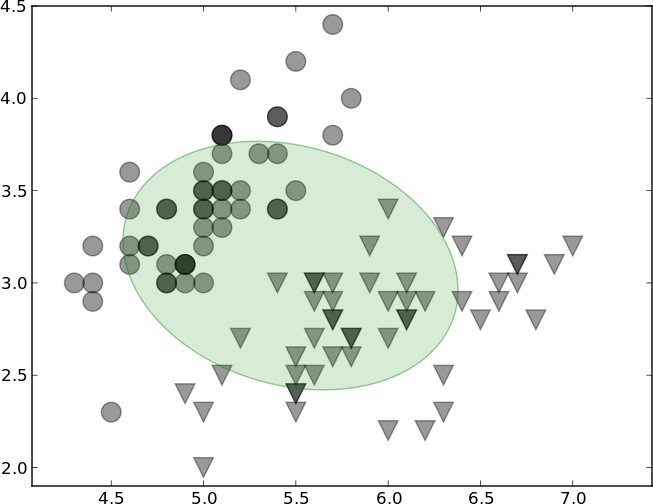}\,\,\,
 \includegraphics[width=0.3\textwidth]{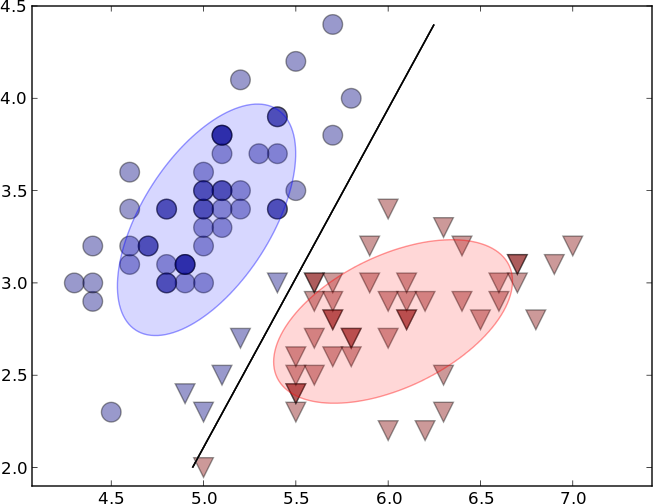}
\end{center}

 \caption{Comparison of $\Sigma_x$ matrices illustrated by ellipsoids, used in kernel projection on
a subset of Iris dataset: in the leftmost picture we used RBF, on the center Mahalanobis RBF and
on the right -based (the black line denotes the  boundary for points being centroids of k-means solution).}
\label{fig:Iris}
\end{figure*}

The basic idea of our approach is to allow the dependence
of the projection type on the local point's neighborhood by transforming each point into some multivariate Gaussian:
$$
x \to \nor(x, \Sigma_x).
$$
On this level of generality it leads to numerically 
complex problems. As we need to calculate some $\Sigma_x$ for every point
in space. and  calculate inverses of some matrices for 
every pair $x,y$ of points from a dataset $X$.

To make the problem more feasible we assume that we have a partition
of the space into $k$ pairwise-disjoint sets 
$$
\R^d=W_1 \cup \ldots \cup W_k \text{ where } W_i \cap W_j = \emptyset \text{ for } {i \neq j} ,
$$
and $\Sigma_x$ depends only on the $x$'s belonging to an element of the partition. 
Consequently we fix $\Sigma_1,\ldots,\Sigma_k$ and consider the feature space
projection
$$
x \to \nor(x,\Sigma_i) \text{ if } x \in W_i, i \in \{1,\ldots,k\}.
$$
Note that we still need to define $\Sigma_i$ for each element
$W_i$. To avoid complex numerical optimization
we simply use the empirical covariance of $X$ restricted to $W_i$:
$$
\Sigma_i:=\cov(X \cap W_i).
$$
The only thing left to consider is how to construct
the partition $W_i$. To simplify this (in general
infinite-dimensional) problem we 
assume \footnote{We used in our experiments also more advanced
partitions, in particular given by GMM, but the results where worse
than obtained by k-means, see discussion in Section~\ref{sec:eval}}
that $W_i$ is given as the k-means partition
of $\R^d$.

Summarizing we define the transformation $x \to \Sigma_x$ by the following steps:
\newline
\begin{tabular}{ll}
i) & choose $k > 0$; \\
ii) & fix  partition $W_1,\ldots,W_k$ of $\R^d$; \\
iii) & define $\Sigma_i:=\cov(X \cap W_i)$; \\
iv) & $\Sigma_x:=\Sigma_i$ if $x \in W_i$.
\end{tabular}  

We illustrate the choice of the feature space projection on the
two-dimensional projection of Iris dataset restricted to two first
classes on Fig.~\ref{fig:Iris}.

It is worth noting that our method is somewhat similar to the ideas behind
metric learning -- the distance in the feature space should be better fitted 
to the geometry of the data. \vrbf{k} algorithm tries to adapt this measure
to the given set of points but without incorporating it inside the optimization process.
This more numerically efficient approach modifies the projection based on the 
neighborhood of the point in space\footnote{We add to some extent the additional kind of
information which is included in the form of feature space projection.}. 

There still remains the question how to choose the
number $k$ and the  partition $W_i$. In the evaluation section we show
that $k=2$ is a good choice for typical datasets, and discuss
some problems and benefits connected with 
fixing $k>2$. Clearly, in general one can select  partition in an arbitrary way, including search through
possible solutions of some optimization problem. In our work we focus on the 
simple and natural construction coming from the k-means clustering of the whole $X$.

\subsection{Generalized RBF Kernel Projection}

Contrary to some models we use the 
existing kernel space, and only modify the 
projection function $\phi$. We define the feature space
as multivariate Gaussians, but contrary to RBF or Mahalanobis RBF
we allow the variation of the Gaussians covariance over the space.
Since we use the Hilbert $L^2(\R^d)$ kernel
space, we do not have to prove that the feature space
projection is correct. The only thing we need is the
formula for the scalar product of two points' projections:
$$
\int \nor(m_1,\Sigma_1)[x] \cdot \nor(m_2,\Sigma_2)[x] dx.
$$
The formula for the $L^2$ scalar product of two Gaussians is known and can be
 easily deduced from the formula for
the sum of two independent normal variables \cite{statystyka}. 
We provide the direct derivation here for the sake of completeness.

\begin{lemma}
Let $m_1,m_2 \in \R^d$ and positive self-adjoint
matrices $\Sm_1,\Sm_2$ be given. We put
\begin{equation*}
 \begin{aligned}
m&=(\Sm_1+\Sm_2)^{-1}(\Sm_1 m_1+\Sm_2 m_2),\\
\Sm&=\Sm_1+\Sm_2,\\
W&=(\Sm_1^{-1}+\Sm_2^{-1})^{-1}.
 \end{aligned}
\end{equation*}
Then
\begin{equation} \label{ee1}
\begin{aligned}
&(x-m_1)^T\Sm_1(x-m_1)+(x-m_2)^T\Sm_2(x-m_2)\\[0.5ex]
=&(x-m)^T\Sm(x-m)+(m_1-m_2)^TW(m_1-m_2).
\end{aligned}
\end{equation}
\end{lemma}

\begin{proof}
One can easily see that the main difficulty lies
in proving that the parts ,,without'' $x$ on left
and right-hand sides of \eqref{ee1} coincide. Now
the left hand side of \eqref{ee1} can be rewritten in the following
form:
\begin{equation} \label{e1}
\begin{array}{l}
-m^T\Sm m+m_1^T\Sm_1m_1+m_2^T\Sm_2m_2 \\[1ex]
= 
-(m_1^T\Sm_1+m_2^T\Sm_2)(\Sm_1+\Sm_2)^{-1} 
(\Sm_1m_1+\Sm_2m_2)\\
\phantom{=} +
m_1^T\Sm_1(\Sm_1+\Sm_2)^{-1}(\Sm_1m_1+\Sm_2m_1)\\
\phantom{=} +
m_2^T\Sm_2(\Sm_1+\Sm_2)^{-1}(\Sm_1m_2+\Sm_2m_2)\\[1ex]
=  m_1'\Sm_1(\Sm_1+\Sm_2)^{-1}\Sm_2(m_1-m_2) \\
\phantom{=}
-m_2'\Sm_2(\Sm_1+\Sm_2)^{-1}\Sm_1(m_1-m_2).
\end{array}
\end{equation}

Clearly
$$
\begin{array}{l}
W=(\Sm_1^{-1}+\Sm_2^{-1})^{-1}=
\Sm_1[(\Sm_1^{-1}+\Sm_2^{-1})\Sm_1]^{-1}= \\[0.5ex]
=\Sm_1(I+\Sm_2^{-1}\Sm_1)^{-1}
=\Sm_1(\Sm_2^{-1}\Sm_2+\Sm_2^{-1}\Sm_1)^{-1} \\[0.5ex]
=\Sm_1(\Sm_2+\Sm_1)^{-1}\Sm_2.
\end{array}
$$
Analogously $W=\Sm_2(\Sm_1+\Sm_2)^{-1}\Sm_1$,
which means that \eqref{e1} reduces to 
$(m_1-m_2)^TW(m_1-m_2)$.
\end{proof}

We recall that for $m \in \R^d$ and positive self-adjoint matrix
$\Sigma$ by $\nor(m,\Sigma)$ we denote the normal
density with mean $m$ and covariance matrix
$\Sigma$, that is
$$
\nor(m,\Sigma)[x]:=\makeitsmaller{\frac{1}{\sqrt{(2\pi)^d\det \Sigma}}}
\exp (-\tfrac{1}{2}\|x-m\|_{\Sigma}^2),
$$
where $\|v\|^2_{\Sigma}$ denotes the square of Mahalanobis
norm of $v$ given by $\|v\|^2_{\Sigma}=v^T\Sigma^{-1}v \for v \in \R^d$.

\begin{proposition}
Let $m_1,m_2 \in \R^d$ and let $\Sigma_1,\Sigma_2$
be positive self-adjoint matrices on $\R^d$.
We have
$$
\int \nor(m_1,\Sigma_1)[x] \cdot \nor(m_2,\Sigma_2)[x]
dx
=\makeitsmaller{\frac{1}{\sqrt{(2\pi)^{d}\det(\Sigma_1+\Sigma_2)}}}
\exp(-\tfrac{1}{2}\|m_1-m_2\|^2_{\Sigma_1+\Sigma_2}).
$$
\end{proposition}

\begin{proof}
We put $\Sm_i=\Sigma_i^{-1}$ and define $m$ as in the previous lemma.
By \eqref{ee1} we have
$$
\begin{array}{c}
\exp(-\tfrac{1}{2}\|x-m_1\|^2_{\Sigma_1}) \cdot \exp(-\tfrac{1}{2}\|x-m_2\|^2_{\Sigma_2}) 
\\[1ex]
=\exp(-\tfrac{1}{2}\|x-m\|^2_{(\Sigma_1^{-1}+\Sigma_2^{-1})^{-1}})
\cdot \exp(-\tfrac{1}{2}\|m_1-m_2\|^2_{\Sigma_1+\Sigma_2}),
\end{array}
$$
which implies that
$$
(2\pi)^d\sqrt{\det \Sigma_1 \det \Sigma_2} \int \nor(m_1,\Sigma_1)[x] \cdot \nor(m_2,\Sigma_2)[x]dx
$$
$$
= \int \exp(-\tfrac{1}{2}\|x-m_1\|^2_{\Sigma_1}) \cdot \exp(-\tfrac{1}{2}\|x-m_2\|^2_{\Sigma_2}) dx
$$
$$
=
\exp(-\tfrac{1}{2}\|m_1-m_2\|^2_{\Sigma_1+\Sigma_2})
\int \exp(-\tfrac{1}{2}\|x-m\|^2_{(\Sigma_1^{-1}+\Sigma_2^{-1})^{-1}}) dx.
$$
Since
\begin{equation*}
 \begin{aligned}
\int \nor(m,(\Sigma_1^{-1}+\Sigma_2^{-1})^{-1})[x] dx &=(2\pi)^{d/2}\sqrt{\det ((\Sigma_1^{-1}+\Sigma_2^{-1})^{-1})}\\
&=\frac{(2\pi)^d\sqrt{\det \Sigma_1 \det \Sigma_2}}{(2\pi)^{d/2}\sqrt{\det(\Sigma_1+\Sigma_2)}}  
 \end{aligned}
\end{equation*}
we obtain the assertion of the proposition.
\end{proof}

\begin{theorem}
\label{theorem:kernel}
Consider a function $x \to \Sigma_x$ and put 
 $$
\hat K_\gamma (x,y) =\makeitsmaller{\frac{1}{\sqrt{(\tfrac{\pi}{\gamma})^d\det(\Sigma_x+\Sigma_y)}}}
\exp(-\gamma\|x-y\|^2_{\Sigma_x+\Sigma_y}).
$$
Then $\hat K_\gamma$ is a valid kernel in the Mercer's sense for each $\gamma>0$.
\end{theorem}

\begin{proof}
 It is a direct consequence of the fact that $\hat K_\gamma$
 is a scalar product in the feature space $\phi(X) \subset L^2(\R^d)$ for feature projection defined by
$$
\phi(x) = \nor(x, (2\gamma)^{-1}\Sigma_x).
$$
\end{proof}

This concept generalizes two mentioned RBF kernels in a natural way.

\begin{observation}
For constant function $\Sigma_x = I$ our approach reduces to the
classical RBF and for $\Sigma_x = \cov(X)$ (or 
for $k=1$ and the trivial -based calculation) to 
the Mahalanobis RBF (up to the scaling factor).
\end{observation}


In practical usage, some of the $\Sigma_x + \Sigma_y$
may be not invertible, and as a consequence we cannot compute $\|x-y\|^2_{\Sigma_x+\Sigma_y}$.
To deal with this case we ca use the typical regularization approach.

\begin{observation} \label{obs:invertible}
If for some $x$ the matrix $\Sigma_x$
is not invertible it is sufficient to replace
$\Sigma_x$ with $(1~-~\varepsilon)\Sigma_x + \varepsilon A$
for any fixed $\varepsilon \in (0,1)$ and positive 
matrix $A$ (then $\Sigma_x$ becomes positive, and
consequently invertible).
\end{observation}

In practice as $A$ we use the covariance 
of the whole dataset $X$, since we want to 
retain the geometry of the data. We also select $\varepsilon$ 
as small as it is reasonably possible in order to preserve
our method characteristics instead of mimicking
the Mahalanobis RBF approach.


One more interesting border case of our kernel is restriction
on the classes of $\Sigma_x$ used, which leads to the computational
complexity reduction.

%
\begin{observation}
We can restrict the set of possible $\Sigma_x$ to some subset of positive self-adjoint
matrices in order to obtain simpler (more efficient) kernel. By restricting
to the radial Gaussians we have $\Sigma_x~=~\sigma_x^2 I$ and consequently kernel formula for fixed $\gamma$ simplifies to
$$ 
\makeitsmaller{\frac{1}{\sqrt{(\tfrac{\pi}{\gamma})^d(\sigma_x^2 + \sigma_y^2)}}} \exp(-\tfrac{\gamma}{\sigma_x^2 + \sigma_y^2 }\left \| x-y \right \|^2) 
\propto 
\makeitsmaller{\frac{1}{\sqrt{(\sigma_x^2 + \sigma_y^2)}}} \exp(-\tfrac{\gamma}{\sigma_x^2 + \sigma_y^2 }\left \| x-y \right \|^2),
$$
which has a computational complexity of standard RBF kernel even though each point can
have its own local variance.
\end{observation}
%
%
\section{Practical considerations}

%
For fixed $\gamma$ we may drop the constant factors from the definition of the kernel in Theorem~\ref{theorem:kernel} which leads to a more numerically efficient formula:
$$
K_\gamma(x,y) = \makeitsmaller{\frac{1}{\sqrt{\det( \Sigma_x + \Sigma_y )}}} \exp(-\gamma \left \| x - y \right \|^2_{\Sigma_x + \Sigma_y}),
$$
which is implemented in Algorithm~\ref{alg:kernelbuilding}. As it has been already shown in Observation~\ref{obs:invertible}, in case the sum of covariances is not invertible, we can simply substitute those matrices with convex combination with the covariance of the whole set (if it is invertible, or with identity matrix otherwise) with some small $\varepsilon$.

\begin{algorithm}

   \caption{\vrbf{k} kernel building }

   \label{alg:kernelbuilding}

\begin{algorithmic}

   \STATE {\bfseries Input:} data $X$, kernel parameter $\gamma$,
   \STATE
   \STATE $X_1,\ldots,X_k \gets$ cluster($X$)
   \STATE $\Sigma_i \gets \cov(X_i)$ {\bfseries for} $i \in \{1,\ldots,k\}$
   \FOR {$i \in \{1,\ldots,k\}$}
    \IF{$\det(\Sigma_i) \leq 0$}
    \STATE $\Sigma_i \gets (1-\varepsilon)\Sigma_i+\varepsilon \cov(X)$  
{\bfseries for small\footnote{}} $\varepsilon \in (0,1)$
	\ENDIF
   \ENDFOR
   \STATE $n_{i,j} \gets \sqrt{1/\!\det(\Sigma_i+\Sigma_j)}$ {\bfseries for} $i,j \in \{1,\ldots,k\}$
   \STATE $S_{i,j} \gets (\Sigma_i+\Sigma_j)^{-1}$ {\bfseries for } $i,j \in \{1,\ldots,k\}$
   \STATE $\mathrm{i}(x) := i $ {\bfseries where} $ x \in X_i$ 
   \STATE $K_\gamma (x,y) := n_{\mathrm{i}(x),\mathrm{i}(y)} \exp(-\gamma (x^T S_{\mathrm{i}(x),\mathrm{i}(y)} y))$
   \STATE {\bfseries return } $K_\gamma$
   
\end{algorithmic}

\end{algorithm}
\footnotetext{In our experiments, we use $\varepsilon = 10^{-10}$, but it can be arbitrary small number which does not lead to numerical problems}

Complexity of our kernel building algorithm depends on the complexity of k-means and the matrix operations (determinant, inversion) applied to the sums of local covariance matrices. In the naive implementation matrix inversion's complexity is $O(d^{2.81})$, where $d$ is the problem's dimensionality, so the whole process adds $O(k^2d^{2.81})$ complexity term to the preprocessing. As we are assuming that $d$ is relatively small (so there is a need of using RBF like kernel), this cost is negligible as this process is needed only once per given data and parameter $\gamma$. It is worth noting that for fixed $k$ the complexity of this step is asymptotically equal to the computation of Mahalanobis RBF kernel. Moreover, \vrbf{k} does not complicate (or deconvexify) the optimization process itself, which is the case in Mahalanobis metric learning SVM and other modifications using the additional optimization routines.

In order to avoid repeated recalculation of the determinants and inversions during cross-validation procedure we can exploit the fact that the only element that is dependent on the choice of $\gamma$ is the exponent value. After easy calculations we arrive at the kernel value conversion formula 
$$
K_{\hat \gamma}(x,y) = n_{\mathrm{i}(x),\mathrm{i}(y)} \exp( \ln(K_{\gamma}(x,y)/n_{\mathrm{i}(x),\mathrm{i}(y)})\tfrac{\hat \gamma}{\gamma}).
$$ 

As the kernel building part of \vrbf{k} does not use labels of samples, we can build upon all the available data, including unlabeled samples as well as data from the testing set. 
Consequently we do not have to run it with each fold during cross-validation separately, but rather cluster the whole data and then train separate SVMs on the data subsets. This feature can be especially exploited when applied in active learning setting \cite{settles2012active} where we have large amount of unlabeled examples.

\section{Evaluation}
\label{sec:eval}

Evaluation was performed using nine datasets from UCI repository~\cite{Bache+Lichman:2013}, briefly summarized in Table~\ref{tab:data}. All points were linearly scaled to the $[0,1]$ interval for fair comparision with regular RBF. As one can see, almost all considered datasets show significant differences in internal geometry between clusters detected by k-means algorithm (sixth column). Only crashes data seems to be quite homogeneous.
\begin{table}[htb]
\centering
\resizebox{\columnwidth}{!}{%
 \begin{tabular}{lrrrrrrr}
 \toprule
 dataset	&	d 	&	n$_\minusOne$ & n$_\plusOne$ 
 & $ \frac{ \left \| \Sigma_1 - I \right \|}{\left \| \Sigma_1  \right \| + \left \| I \right \|} $ 
 & $ \frac{ \left \| \Sigma_2 - I \right \|}{\left \| \Sigma_2  \right \| + \left \| I \right \|} $
 & $ \frac{ \left \| \Sigma_2 - \Sigma_1 \right \|}{\left \| \Sigma_2  \right \| + \left \| \Sigma_1 \right \|} $ 
 & $ \frac{ \left \| (\Sigma_2 + \Sigma_1) - \Sigma \right \|}{\left \| \Sigma_2 + \Sigma_1 \right \| + \left \| \Sigma \right \|} $  \\
 \midrule
 
 australian & 14 & 329 & 361 & 0.574 & 0.549 & 0.151 & 0.363 \\
bank & 4 & 462 & 910 & 0.890 & 0.828 & 0.466 & 0.189 \\
breast cancer & 10 & 453 & 230 & 0.890 & 0.514 & 0.760 & 0.459 \\
crashes & 20 & 270 & 270 & 0.992 & 0.992 & 0.003 & 0.333 \\
diabetes & 8 & 515 & 253 & 0.854 & 0.824 & 0.295 & 0.328 \\
fourclass & 2 & 404 & 458 & 0.701 & 0.664 & 0.116 & 0.319 \\
heart & 13 & 129 & 141 & 0.469 & 0.507 & 0.327 & 0.344 \\
liver-disorders & 6 & 293 & 52 & 0.909 & 0.741 & 0.636 & 0.458 \\
splice & 60 & 592 & 408 & 0.331 & 0.361 & 0.265 & 0.328 \\

 \bottomrule
 \end{tabular}
 }
 \caption{Characteristics of used datasets. $\Sigma_1$ and $\Sigma_2$ denotes the covariances of first and second cluster found by k-means (with $k=2$).}
 \label{tab:data}
 \end{table}
Experiments were performed using code written in Python with use of \texttt{scikit-learn} library. K-means algorithm was seeded with K-means++~\cite{kmeanspp} 10 times and clustering yielding the smallest energy was selected. All data (including test cases) was used during the clustering step (as unlabeled examples). All experiments were performed in 10-fold cross-validation mode.

We start our evaluation with reporting the best accuracy obtained by all tested models. Table~\ref{tab:accuracy} shows that proposed method achieves similar results to the RBF kernel and Mahalanobis RBF. In some cases \vrbf{k} behaves significantly better (australian, diabetes, breast-cancer) and for some worse than referencing kernels (crashes, heart, liver-disorders). These results are the consequence of many aspects -- including the choice of the simplest clustering algorithm, naive empirical covariance estimation. They show, however, that in terms of achieved overall accuracy using \vrbf{k} leads to comparable results to RBF/mRBF based classification. To show, that our approach is fundamentally different from applying k-means as a data partitioning scheme, and training separate mRBF based models in each cluster, we also report results of such model (dentoed as m$_k$RBF, meaning that it first runs k-means and in each cluster trains separate cluter's covariance based SVM).
\begin{table}[htb]
\centering
\resizebox{\columnwidth}{!}{%
 \begin{tabular}{lrrrrrrrr}
 \toprule 
 dataset 		& RBF 	& mRBF 	& \vrbf{2} & \vrbf{3} & \vrbf{4} & m$_2$RBF & m$_3$RBF & m$_4$RBF\\
 \midrule 
 australian 		& 0.862 	& 0.856	 	& 0.872 	& \w0.875 	& 0.859		& 0.838 	& 0.801 	& 0.797\\
 bank 			& \w1.000 	& \w1.000 	& \w1.000	& \w1.000 	& \w1.000	& \w1.000 	& \w1.000	& \w1.000\\
 breast-cancer 		& 0.972 	& 0.971 	& \w0.975 	& 0.972 	& 0.972 	& 0.955 	& 0.953		& 0.932 \\
 crashes 		& \w0.952 	& 0.948 	& 0.939 	& 0.943 	& 0.930 	& 0.939 	& 0.944		& 0.931\\
 diabetes 		& 0.755 	& 0.760 	& \w0.772 	& 0.768 	& 0.758 	& 0.758 	& 0.764		& 0.743 \\
 fourclass 		& \w0.996 	& \w0.996 	& \w0.996 	& \w0.996 	& \w0.996	& \w0.996 	& \w0.996 	& \w0.996\\
 heart 			& \w0.844 	& \w0.844 	& 0.826 	& 0.819 	& 0.789 	& 0.778 	& 0.789		& 0.789\\
 liver-disorders 	& 0.731 	& \w0.734 	& 0.728 	& 0.722 	& 0.720 	& \w0.734	& 0.722		& 0.733\\
 splice 		& \w0.893 	& 0.868 	& \w0.893 	& 0.888 	& 0.871 	& 0.892		& 0.883		& 0.874 \\
 \bottomrule
  
 \end{tabular}
 }
 \caption{Comparision of accuracy obtained by different kernels. For \vrbf{k} k-means is used as the clustering technique}
\label{tab:accuracy}
 
\end{table}

We investigated how different clustering techniques behave in such task. We performed experiments for Gaussian Mixture Model (GMM) and Dirichlet Process Gaussian Mixture Model (DPGMM) with number of clusters varying from $2$ to $4$. Table~\ref{tab:gmm} shows differences between accuracy obtained for k-means based method and Gaussian Mixture Models. One can make at least two important observations here. First, k-means performs surprisingly well as compared to more advanced clustering methods. Second, for some datasets (like liver-disorders) GMM based solution brings considerable increase in classification quality (which outperforms also RBF and mRBF, refer to Table~\ref{tab:accuracy}). This may lead to the conclusion, that different types of clustering methods can exploit various types of knowledge. The choice of a good clustering technique requires an additonal analysis of data, internal cross-validation etc. so in further parts of our investigations we focus only on k-means based approach to show its wide 
applicability. However, reader should bear in mind, that different methods are possible. 

\begin{table}[htb]
\centering
\resizebox{\columnwidth}{!}{%
 \begin{tabular}{lrrrrrrr}
 \toprule 
 dataset 		&\vrbf{2} & GMM$_2$ & GMM$_3$ & GMM$_4$  & DPGMM$_2$ & DPGMM$_3$ & DPGMM$_4$\\
 \midrule 
 australian 		& \w0.872 & \w0.872 & 0.768 & 0.852 & 0.868 & 0.868 & 0.868 \\
 bank 			& \w1.000 & \w1.000 &\w1.000 &\w1.000 &\w1.000 &\w1.000 &\w1.000  \\
 breast-cancer 		& \w0.975 & 0.971   & 0.972 & 0.965 & 0.971 & 0.971 & 0.971 \\
 crashes 		& 0.939   & \w0.951 & 0.943 & 0.933 & 0.946 & 0.946 & 0.946  \\
 diabetes 		& \w0.772 & 0.744   & 0.755 & 0.736 & 0.767 & 0.767 & 0.767 \\
 fourclass 		& \w0.996 & \w0.996 & \w0.996 & \w0.996 & \w0.996 & \w0.996 & \w0.996  \\
 heart 			& 0.826 & 0.815 & 0.807 & 0.815 & \w0.844 & \w 0.844 & \w0.844\\
 liver-disorders 	& 0.728 & 0.739 & 0.737 & \w0.748 & 0.722 & 0.722 & 0.722\\
 splice 		& \w0.893 & 0.877 & 0.887 & 0.864 & 0.866 & 0.866 & 0.866 \\
 \bottomrule
  
 \end{tabular}
 }
 \caption{Comparision of accuracy of proposed method with various clustering methods. K-means is used for \vrbf{2}}
 \label{tab:gmm}
\end{table}

The most interesting effect of using the proposed method is easier metaparameters selection. In practice, many applied researchers (for example in cheminformatics  \cite{Wang2008hergSVM,Cong2009TACESVM,Sakiyama2008stabilitySVM}) neglect the metaparameters optimization and use its default values. In most of existing SVM libraries (including \texttt{libSVM}, \texttt{WEKA}), the default value of the $C$ metaparameter is $1$. Table~\ref{tab:accuracyC1} shows accuracy obtained by considered models once we narrow down to the optimization of only $\gamma$. \vrbf{2} obtaines significantly better results than both RBF and mRBF in most cases. It achieves worse performance than RBF kernel only in two tests, where also mRBF behaved worse, which simply shows, that in these datasets, covariance based geometry is not a good kernel building base.

\begin{table}[htb]
\centering
\resizebox{\columnwidth}{!}{%
 \begin{tabular}{lrrrrrrrr}
 \toprule 
 dataset 		& RBF 	& mRBF 	& \vrbf{2} & \vrbf{3} & \vrbf{4}   & m$_2$RBF & m$_3$RBF & m$_4$RBF\\
 \midrule 
 australian 		& 0.855 & 0.857 & \w0.862 & \w0.862 & 0.859 & 0.762 & 0.801 & 0.791\\
 bank 			& 0.978 & \w1.000 & \w1.000 & \w1.000 & \w1.000& 0.999 & \w1.000 & \w1.000 \\
 breast-cancer 		& 0.968 & 0.960 & \w0.971 & \w0.971 & 0.968 & 0.887 & 0.886 & 0.884 \\ 
 crashes 		& \w0.946 & 0.935 & 0.939 & 0.943 & 0.930 & 0.922 & 0.922 & 0.922 \\
 diabetes 		& 0.751 & 0.760 & \w0.772 & 0.759 & 0.756 & 0.758 & 0.763 & 0.730\\
 fourclass 		& 0.731 & 0.778 & 0.778 & 0.841 & 0.849 & 0.852 & \w0.897 & 0.891 \\
 heart 			& \w0.844 & 0.830 & 0.815 & 0.789 & 0.789 & 0.744 & 0.778 & 0.737\\
 liver-disorders 	& 0.580 & 0.708 & \w0.728 & 0.708 & 0.696 & 0.717 & 0.685 & 0.713\\
 splice 		& 0.833 & 0.840 & \w0.893 & 0.888 & 0.871 & 0.883 & 0.868 & 9.842\\
 \bottomrule
  
 \end{tabular}
 }
 \caption{Comparision of accuracy obtained by different kernels when using (default) parameter value $C=1$. For \vrbf{k} k-means is used as the clustering technique}
\label{tab:accuracyC1}
\end{table}

To further invesigate this phenomen we also performed experiments with very limited search of parameters. We fixed $C=1$ and searched through just $3$ values of $\gamma$ ($\gamma \in \{ 10^i,10^{i+1},10^{i+2} \}$ for $i \in \{0,-1,-2,-3,-4,-5\}$). In Table~\ref{tab:wins} we report percentage of wins (times that \vrbf{k} obtained better accuracy) for each of such small tests. One can notice, that in such small metaparameters ranges, proposed method outperformed RBF and mRBF in almost all cases. Such results are of great importance in applications where we often have limited resources and many internal cross-validation based parameters selection are not possible. One such application could be the active learning scenario, or the hierarchical/ensamble based models.

\begin{table}[htb]
\centering
\resizebox{\columnwidth}{!}{%
 \begin{tabular}{lrrrrrr}
 \toprule 
			& \vrbf{2} vs & &  \vrbf{3} vs & &  \vrbf{4} vs & \\
 dataset 		& RBF 	& mRBF 	&  RBF 	& mRBF 	 & RBF 	& mRBF 	 \\
 \midrule 
 australian 		& \w 67\%  & \w 67\% 	& 50\%  & 50\% 	& 50\% 	& 50\% \\
 bank 			& \w 100\% & \w 83\% 	& \w 100\% & \w 83\% 	& \w 100\% & \w 83\% \\
 breast-cancer 		& \w 67\%  & \w 100\% & \w 67\%  & \w 100\% & \w 67\% 	& \w 100\%  \\
 crashes 		& 23\%  & \w 67\% 	& \w 67\% 	& \w 100\% & \w  67\% & \w 67\% \\
 diabetes 		& \w 100\% & \w 100\% & \w 83\% 	& 50\% 	& \w 83\% 	& 50\% \\
 fourclass 		& \w 100\% & \w 83\% 	& \w 100\% & \w 83\% 	& \w 100\% & \w 100\% \\
 heart 			& 50\% 	& \w 67\% 	& 50\% 	& \w 67\% 	& 50\% 	& \w 67\%\\
 liver-disorders 	& \w 83\% 	& \w 83\% 	& \w 100\% & \w 83\% 	& \w 100\% & \w 83\% \\
 splice 		& \w 83\% 	& \w 83\% 	& \w 100\% & \w 100\% & \w 100\% & \w 100\%  \\
 \bottomrule  
 \end{tabular}
 }
 \caption{Percentage of 3 element wide search for $\gamma$ values with fixed $C=1$ for which given \vrbf{k} (based on k-means) achieved higher accuracy than corresponding kernel. Bolded values indicate values bigger than 50\%}
\label{tab:wins}
 
\end{table}

\begin{figure}[htb]
\begin{center}
\includegraphics[width=0.3\textwidth]{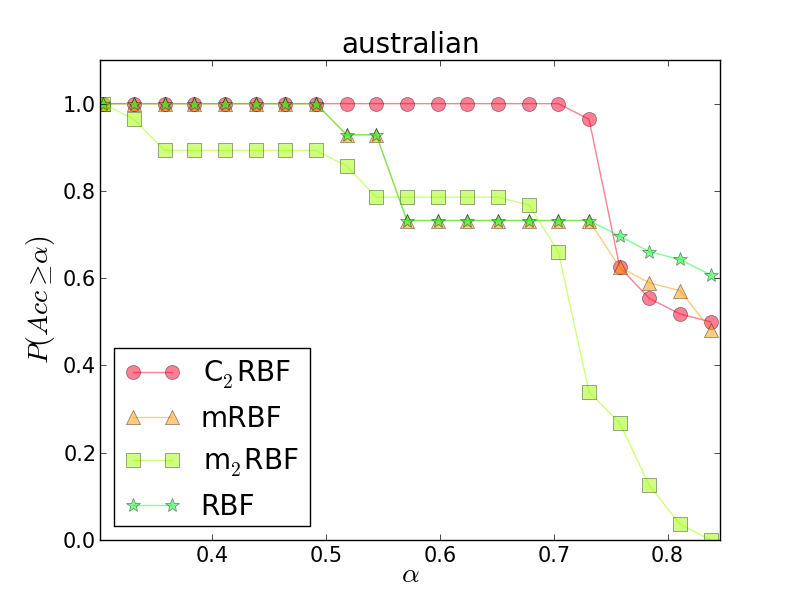}
\includegraphics[width=0.3\textwidth]{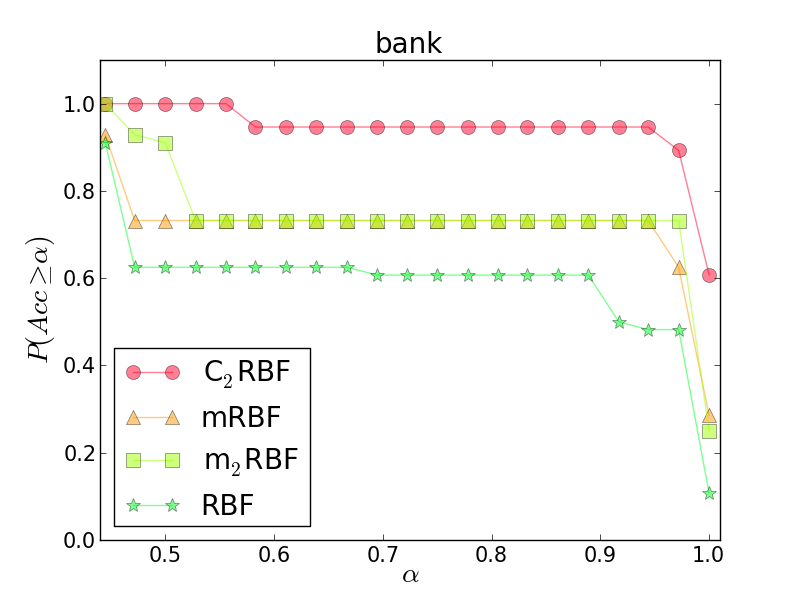}
\includegraphics[width=0.3\textwidth]{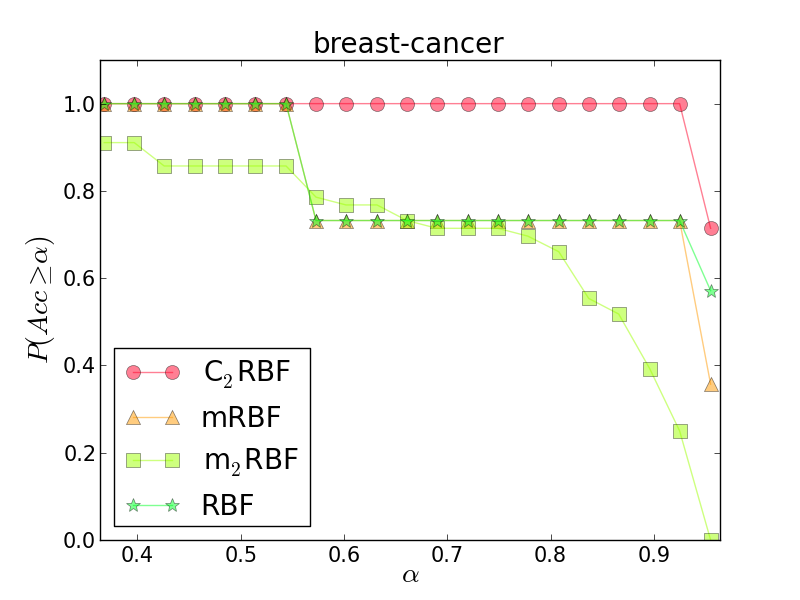}

\includegraphics[width=0.3\textwidth]{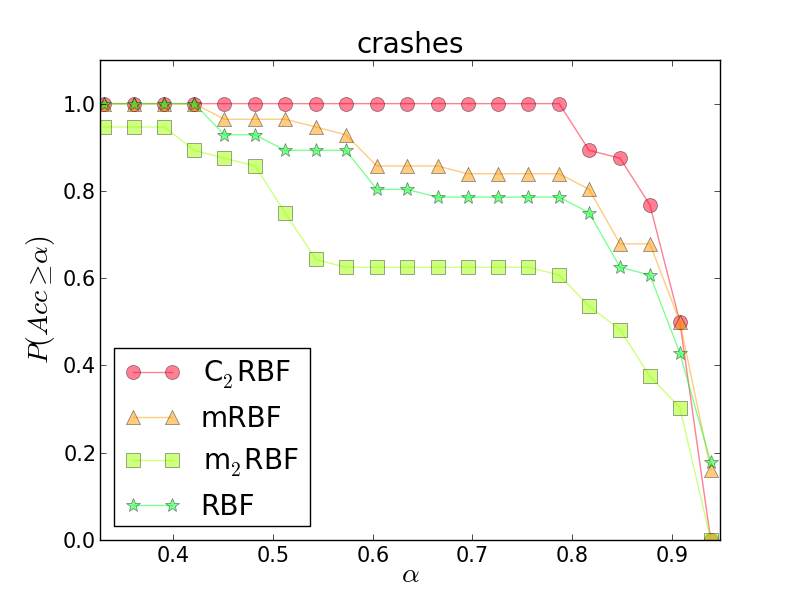}
\includegraphics[width=0.3\textwidth]{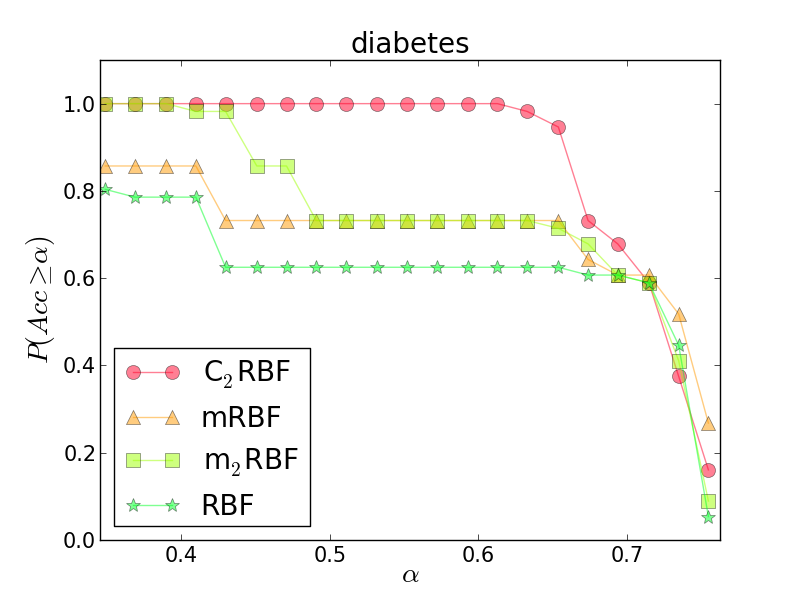}
\includegraphics[width=0.3\textwidth]{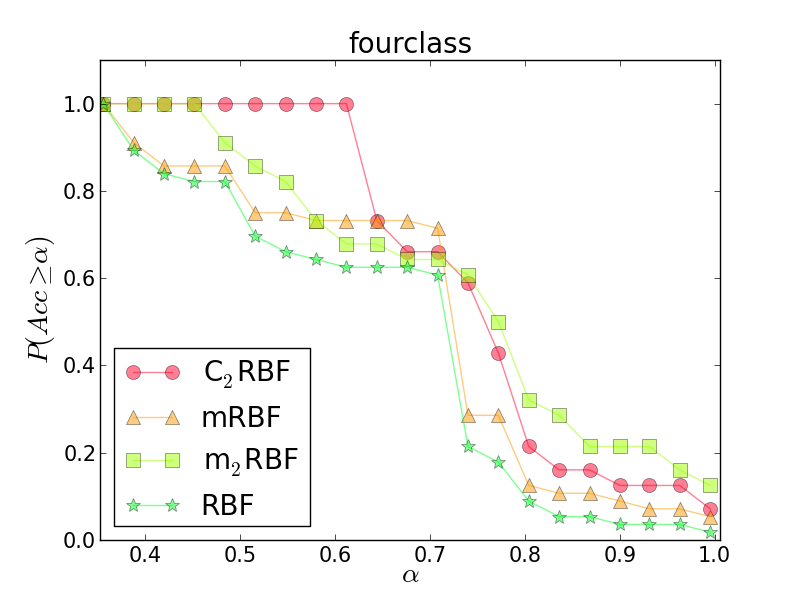}

\includegraphics[width=0.3\textwidth]{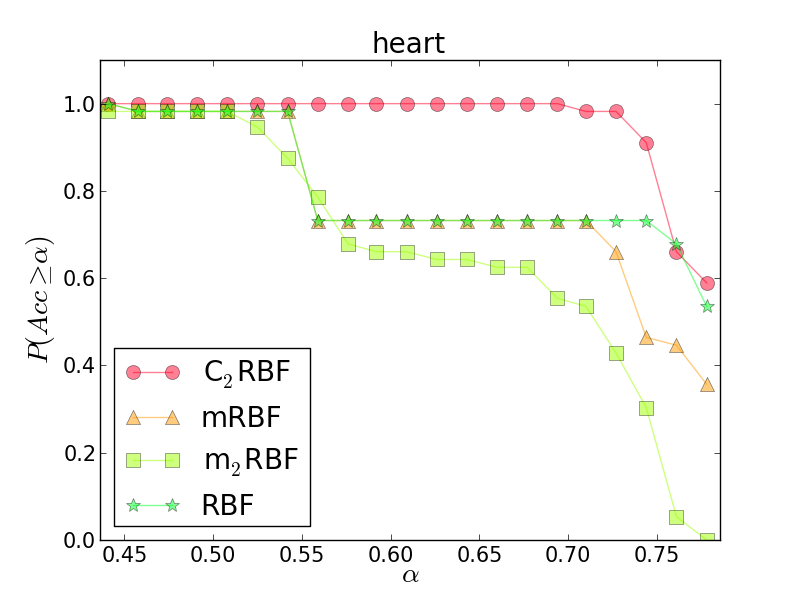}
\includegraphics[width=0.3\textwidth]{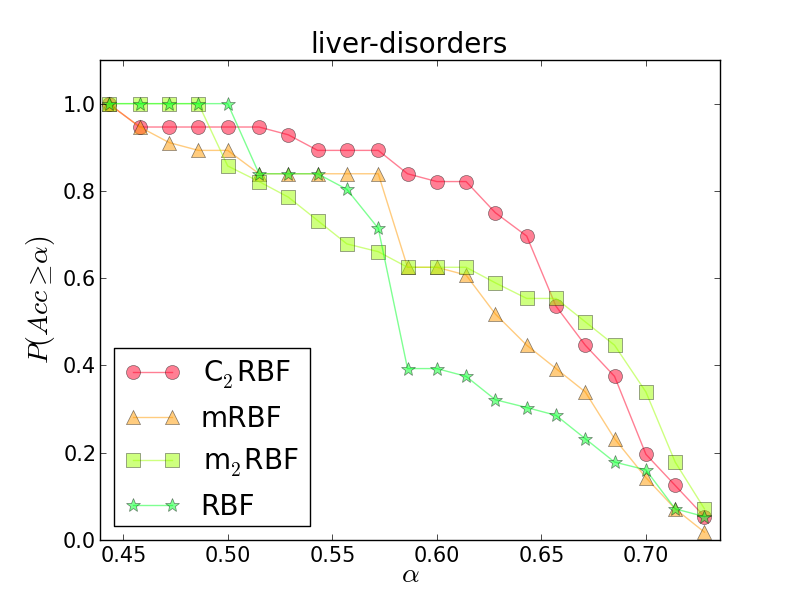}
\includegraphics[width=0.3\textwidth]{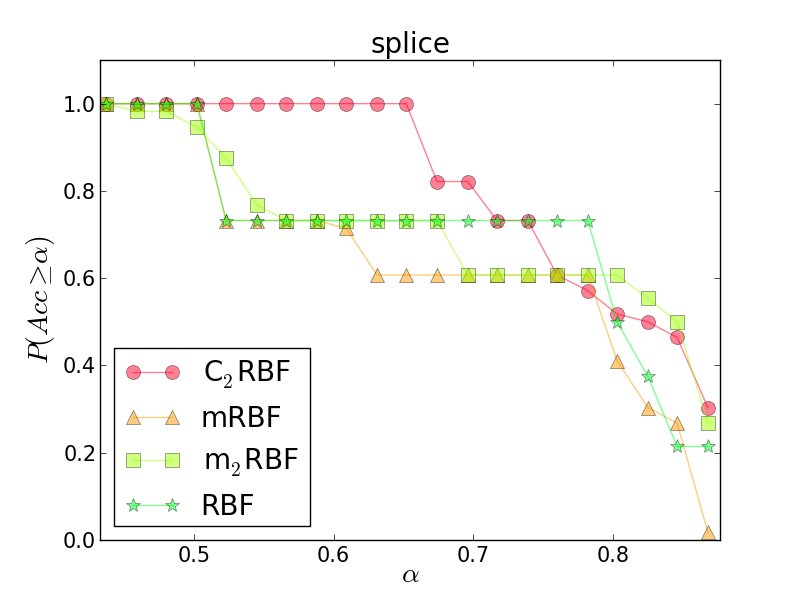}
\end{center}
\caption{Visualizaion of $P_f(\alpha)$ for conducted experiments for \vrbf{2} based on k-means algorithm, regular RBF kernel and mRBF.}
\label{fig:prob}
\end{figure}

\begin{table}[htb]
\centering
\resizebox{\columnwidth}{!}{%
 \begin{tabular}{lrrrrrrrr}
 \toprule 
 dataset 		& RBF 	& mRBF 	& \vrbf{2} & \vrbf{3} & \vrbf{4}  & m$_2$RBF & m$_3$RBF & m$_4$RBF\\
 \midrule 
 australian 		& 0.424 & 0.421 & \w0.475 & 0.404 & 0.435 & 0.302 & 0.313 & 0.284\\
 bank 			& 0.328 & 0.400 & 0.526 & 0.530 & \w0.531 & 0.414 & 0.413 & 0.408\\
 breast-cancer 		& 0.478 & 0.475 & \w0.583 & 0.581& 0.579& 0.401 & 0.396 & 0.394\\
 crashes 		& 0.485 & 0.513 & 0.569 & \w0.578& 0.562& 0.401 & 0.411 & 0.412\\
 diabetes 		& 0.253 & 0.288 & \w0.361 & 0.344 & 0.347 & 0.302 & 0.313 & 0.284 \\
 fourclass 		& 0.290 & 0.329 & 0.400 & \w0.417 & 0.406 & 0.385 & 0.359 & 0.367 \\
 heart 			& 0.331 & 0.320 & \w0.387 & 0.382 & 0.375 & 0.285 & 0.269 & 0.272\\
 liver-disorders 	& 0.160 & 0.178 & 0.209 & \w0.213 & 0.204 & 0.190 & 0.186 & 0.183\\
 splice 		& 0.306 & 0.280 & 0.354 & \w0.386 & 0.360 & 0.308 & 0.307 & 0.293 \\
 \bottomrule
  
 \end{tabular}
 }
 \caption{Comparision of AUC of $P_f(\alpha)$ obtained by different kernels. For \vrbf{k} k-means is used as the clustering technique}
 \label{tab:auc}
\end{table}

Observe that in practice, even if we perform metaparameters optimization, we never find the real optimum (in terms of tunable
parameters) and
therefore the increase in resolution of the grid
results in finding better classification results. Thus the 
comparison of two SVM-based classification methods with
just comparing the best result found on the grid can be misleading.
We propose a measure which in our opinion is
 more reliable -- estimation of
the probability of finding  results which 
are better than a fixed parameter value $\alpha$.

Consider the typical case in RBF SVM when our function
depends on two parameters $C$ and $\gamma$.
Let us fix a grid $G$ (Cartesian product of considered $C$'s
and $\gamma$'s) and consider the function
$$
P_{f}(\alpha):=\prob\{f(C,\gamma)\geq \alpha: (C,\gamma) \in G\}.
$$
The above function measures the probability of finding
results better then $\alpha$. 
As this kind of measure better exploits the model's
ability to work well with limited grid size, it directly
corresponds to its applicability in training time limited scenarios.
It is of great practical importance for the models
which are build from many models (ensembles, hierarchical models)
as well as in the active learning setting, when
one has to retrain it repeatedly. We approximate this probability
by the fraction of parameters pairs in the considered grid,
which yield results at least $\alpha$.
$$
P_{f}(\alpha) \approx \hat P_f( \alpha ) := \frac{ \# \{ (C,\gamma) \in G : f(C,\gamma)\geq \alpha \} } { \# G }.
$$

Plots of corresponding $\hat P_f$ functions in Fig.~\ref{fig:prob} illustrate that RBF and mRBF kernels are very similar in context of how hard is to find the parameters yielding good results. Also m$_2$RBF behaves in a very similar fashion, yielding in most cases -- results between the one given by RBF and mRBF. In the same time \vrbf{2} offers noticeably higher probability of yielding comparable results. This observation is purely empirical and the justification of this phenomenon remains for us an open question.

Areas under the $\hat P_f$ curves (AUC) are shown in Table~\ref{tab:auc}. For simplicy, we consider only curves for $\alpha$ at least as big, as the worst score achieved by all models (as $P_f(\alpha)$ for smaller values is constantly equal to $1$ for each model).  In all conducted experiments, \vrbf{2} shows significant improvement over competitive approaches, confirming our claim, that \vrbf{k} based on k-means can be used to simplify the process of metaparameters selection.

\section{Conclusions}

In this paper we have presented the method of including information regarding local problem's geometry inside the definition of the Gaussian kernel. From theoretical point of view proposed method leads to the correct kernel in the Mercer's sense which is based on feature space projection transforming data points into various multivariate Gaussian density functions. 

From practical perspective, our method's kernel building complexity is asymptotically equivalent to the Mahalanobis RBF kernel and is similarly cheap in computation during classification. Obtained results show that our method behaves similarly to RBF and mRBF kernels. However, \vrbf{k} yields better results with higher probability in terms of selecting the typical SVM parameters, $C$ and $\gamma$.

We have also shown empirically that proposed approach is fundamentally different from splitting the problem into subproblems using some clustering method and building separate model for each of them. \vrbf{k} uses clustering in order to augment the data representation with additional knowledge instead of reducing the amount of information available. This shows the conceptual distinction of our approach from previous models.

It is also worth noting that even though we used \mbox{k-means} in our experiments, proposed method can be seen as more general framework, where assignment of $\Sigma_x$ can be the result of an arbitrary complex process. In particular, it would be interesting to further investigate space partitioning given by other supervised, linear classifiers instaed of clustering methods.

\bibliographystyle{spmpsci}

\end{document}